\def\notes{1}

\documentclass[11pt]{article}

\usepackage[english]{babel}
\usepackage[utf8x]{inputenc}
\usepackage[T1]{fontenc}

\usepackage[margin=1.0in]{geometry}

\usepackage{amsmath,amssymb,amsthm,amsmath,mathtools}
\usepackage{multirow}
\usepackage{dsfont}
\usepackage{graphicx}
\usepackage{algorithm}
\usepackage[noend]{algorithmic}
\usepackage[makeroom]{cancel}
\usepackage{footnote}
\usepackage[usenames,dvipsnames]{color}
	\definecolor{mygray}{gray}{0.4}
	\definecolor{DarkGreen}{rgb}{0.2,0.6,0.2}
	\definecolor{DarkRed}{rgb}{0.6,0.2,0.2}
	\definecolor{DarkBlue}{rgb}{0.2,0.2,0.6}
	\definecolor{DarkPurple}{rgb}{0.4,0.2,0.4}
\usepackage[colorlinks=true, allcolors=DarkBlue]{hyperref}

\usepackage[colorinlistoftodos]{todonotes}
\usepackage{verbatim}
\usepackage{microtype}
\usepackage{authblk}
\usepackage{framed}

\usepackage{scrextend}
\usepackage{array}
\usepackage{tabto}

\ifnum\notes=1
\newcommand{\mynote}[1]{\marginpar{\tiny\sf #1}}
\else
\newcommand{\mynote}[1]{}
\fi

\newcommand{\ex}[2]{{\ifx&#1& \mathbb{E} \else \underset{#1}{\mathbb{E}} \fi \left[#2\right]}}
\newcommand{\pr}[2]{{\ifx&#1& \mathbb{P} \else \underset{#1}{\mathbb{P}} \fi \left[#2\right]}}

\newcommand{\prob}{\mathbb{P}}

\newcommand{\eps}{\varepsilon}
\newcommand{\half}{\frac{1}{2}}

\newcommand{\hm}{\hat{\mu}}

\newcommand{\estmean}{\textsc{EstMean}}
\newcommand{\findbest}{\textsc{FindBest}}

\usepackage{cleveref}
\newtheorem{lem}{Lemma}

\newtheorem{defn}[lem]{Definition}

\newtheorem{thm}[lem]{Theorem}

\newtheorem{clm}[lem]{Claim}
\newtheorem*{clm*}{Claim}

\newtheorem*{crly*}{Corollary}

\title{Skyline Identification in Multi-Armed Bandits}
\author{Albert Cheu\thanks{\href{mailto:cheu.a@husky.neu.edu}{\texttt{cheu.a@husky.neu.edu}}}}
\author{Ravi Sundaram\thanks{Supported in part by NSF grant CCF-1535929.  \href{mailto:koods@ccs.neu.edu}{\texttt{koods@ccs.neu.edu}}}}
\author{Jonathan Ullman\thanks{Supported in part by NSF grant CCF-1718088.  \href{mailto:jullman@ccs.neu.edu}{\texttt{jullman@ccs.neu.edu}}}}
\affil{College of Computer and Information Science \\ Northeastern University}

\begin{document}
\maketitle

\begin{abstract}
We introduce a variant of the classical PAC multi-armed bandit problem.  There is an ordered set of $n$ arms $A[1],\dots,A[n]$, each with some stochastic reward drawn from some unknown bounded distribution.  The goal is to identify the \emph{skyline} of the set $A$, consisting of all arms $A[i]$ such that $A[i]$ has larger expected reward than all lower-numbered arms $A[1],\dots,A[i-1]$.  We define a natural notion of an $\eps$-approximate skyline and prove matching upper and lower bounds for identifying an $\eps$-skyline.  Specifically, we show that in order to identify an $\eps$-skyline from among $n$ arms with probability $1-\delta$,
$$
\Theta\bigg(\frac{n}{\eps^2} \cdot \min\bigg\{ \log\bigg(\frac{1}{\eps \delta}\bigg), \log\bigg(\frac{n}{\delta}\bigg) \bigg\} \bigg)
$$
samples are necessary and sufficient. When $\eps \gg 1/n$, our results improve over the na\"ive algorithm, which draws enough samples to approximate the expected reward of every arm; the algorithm of (Auer et al., AISTATS'16) for Pareto-optimal arm identification is likewise superseded. Our results show that the sample complexity of the skyline problem lies strictly in between that of best arm identification (Even-Dar et al., COLT'02) and that of approximating the expected reward of every arm.
\end{abstract}

\vfill
\newpage

\tableofcontents
\vfill
\newpage

\section{Introduction}
In (one formulation of) the classical multi-armed bandit framework, we are given a set of $n$ \emph{arms} to choose from, each of which results in some stochastic reward, and our goal is to find an arm whose expected reward is as large as possible, while drawing as few samples from the arms as possible.  One representative example of this framework comes from drug testing, where there are many available drugs, each of which has some unknown probability of success, and the goal is to find the most effective drug.  We introduce a variant of this problem where the arms are sorted, and the goal is to find all arms $A[i]$ whose expected rewards is as large as possible among the first $i$ arms.  In our representative example, the drugs may have different costs, and we want to produce a list of all drugs that are the most effective for their given cost.  In other words, we want to compute the set of arms $A[i]$ such that there is no $i' \leq i$ with higher expected reward.  Borrowing terminology from the database community, we call this the \emph{skyline} of the set of arms~\cite{BKS01}.  In this work, we define a natural PAC (\emph{Probably Approximately Correct}) version of this problem that we call \emph{$\eps$-skyline identification}, and prove worst-case sample complexity bounds that are optimal up to constant factors in terms of all of the relevant parameters. 

We remark that the problem we study can be cast as a special case of a much more general version of the Pareto-optimal arm identification problem that was recently introduced by Auer et al.~\cite{Auer}, which we discuss in more detail in Section~\ref{sec:relatedwork}.  Our results yield optimal sample complexity bounds for this special case, and may be useful for obtaining optimal sample complexity for other variants of Auer et al.'s model, and, more generally, with obtaining optimal sample complexity for other variants of the \emph{best-arm identification} problem.

\subsection{Model and Results}

Like all multi-armed bandit models, our problem begins with a set of $n$ \emph{arms} $A = \{A[1],\dots,A[n]\}$.  Each arm $A[i]$ represents an probability distribution over $[0,1]$ with unknown mean $\mu[i] = \mathbb{E}[A[i]]$.  We will sometimes abuse terminology and associate an arm $A[i]$ simply with its index $i$.  We can obtain a \emph{sample} from any arm $A[i]$, which we also refer to as \emph{pulling} arm $A[i]$.

A general objective is to learn useful information about the set of arms using as few samples as possible.  Our specific goal is to compute an $\eps$-skyline, which we define as follows.
\begin{defn}[$\eps$-skyline] \label{def:runningmax}
A set $S \subseteq [n]$ is an \emph{$\eps$-skyline of $A$} if the following conditions hold:
\begin{enumerate}
\item For any arm index $t \notin S$, the largest $s \in S$ such that $s < t$ satisfies $\mu[s] \geq \mu[t] - \eps$.
\item For every $s \in S$, and every $t \leq s$, $\mu[s] \geq \mu[t] - \eps$.
\end{enumerate}
\end{defn}
Before we proceed, we discuss how this definition is a natural notion of an approximation skyline.  An arm $t \in [n]$ is in the exact skyline if and only if $\mu[t] \geq \mu[s]$ for every $s < t$.  Suppose that an arm $t$ is convincingly in the skyline, meaning $\mu[t] \geq \mu[s] + \eps$ for every $s < t$.  Then the first condition asserts that $t \in S$.  On the other hand, suppose that an arm $t$ is convincingly not in the skyline, meaning there exists $s < t$ such that $\mu[t] \leq \mu[s] - \eps$.  Then the second condition asserts that $t \not\in S$.  Although the formulation of Definition~\ref{def:runningmax} is a slightly indirect way of capturing this intuitive notion, it will be more convenient to work with in our analysis.

Let's consider a na\"ive approach to identifying an $\eps$-skyline: draw samples from each arm to obtain an empirical estimate $\hat\mu[i] \approx \mu[i]$ and compute the exact skyline of the estimates.  If we want to draw enough samples to obtain an $\eps$-skyline with probability $1-\delta$, then by standard concentration and anti-concentration arguments it is necessary and sufficient to draw $\Theta(\frac{n}{\eps^2} \log\frac{n}{\delta})$ samples in total.  Our first result is an algorithm that improves on this na\"ive approach in the regime where $\eps$ is not too small compared to $n$.
\begin{thm}
\label{thm:main}
For $\eps, \delta \in (0,1)$, there is an algorithm that gets sample access to a set of $n$ arms $A = \{A[1],\dots,A[n]\}$ with rewards in $[0,1]$ and, with probability at least $1-\delta$, draws $O(\frac{n}{\eps^2} \log\frac{1}{\eps \delta})$ samples to return an $\eps$-skyline for $A$ with size at most $O(1/\eps)$.
\end{thm}
This result shows that finding an approximate skyline can be strictly easier than estimating the payoff of every arm.  In particular, when $\eps,\delta$ are constants, the na\"ive algorithm requires $O(n \log n)$ samples, whereas the algorithm in Theorem~\ref{thm:main} requires just $O(n)$ samples.  We give an overview of our algorithm in Section~\ref{sec:algorithm_overview}. 

Combining the na\"ive algorithm with Theorem~\ref{thm:main} yields a sample complexity upper bound of
$$
O\bigg(\frac{n}{\eps^2} \cdot \min\bigg\{ \log\bigg(\frac{1}{\eps \delta}\bigg), \log\bigg(\frac{n}{\delta}\bigg) \bigg\} \bigg).
$$
In order to show this is optimal let's consider a na\"ive lower bound on the sample complexity of $\eps$-skyline identification.  The largest-numbered arm in any $\eps$-skyline is always within $\eps$ of the maximum expected reward among all the arms, thus finding an $\eps$-skyline can be no easier than identifying an arm with nearly maximal payoff.  Combining the algorithm of Even-Dar, Mannor, and Mansou~\cite{Even-Dar} with the lower bound of Mannor and Tsitsiklis~\cite{MannorTsi}, the sample complexity of this latter problem is $\Theta(\frac{n}{\eps^2} \log \frac{1}{\delta})$.  However, this does not yield a tight lower bound for $\eps$-skyline identification.  We observe that, when $\eps \geq \frac{1}{n}$, an algorithm for $\eps$-skyline identification can be used to solve $\Omega(1/\eps)$ independent $\eps$-best arm identification problems, each on $\Omega(\eps n)$ arms.  Formalizing this intuition yields the following tight lower bound, and shows that approximate skyline identification is strictly harder than approximate best arm identification.
\begin{thm}
\label{thm:lowerbound}
Fix $\eps, \delta$ smaller than some absolute constants and $\eps \geq 1/n$.  Suppose there is an algorithm such that for \emph{every} set of $n$ arms, with probability at least $1-\delta$, the algorithm returns an $\eps$-skyline.  Then this algorithm must draw $\Omega(\frac{n}{\eps^2} \log\frac{1}{\eps \delta})$ samples.
\end{thm}
Note that the condition $\eps \geq 1/n$ is necessary for the lower bound to hold, since the na\"ive algorithm gives an upper bound of $O(\frac{n}{\eps^2} \log \frac{n}{\delta})$ samples for every $\eps > 0$. Theorem~\ref{thm:lowerbound} immediately implies that the na\"ive algorithm is optimal up to constant factors when $\eps < 1/n$.  Thus, by combining the na\"ive algorithm with Theorems~\ref{thm:main} and~\ref{thm:lowerbound} we obtain the sample complexity lower bound of 
$$
\Omega\bigg(\frac{n}{\eps^2} \cdot \min\bigg\{ \log\bigg(\frac{1}{\eps \delta}\bigg), \log\bigg(\frac{n}{\delta}\bigg) \bigg\} \bigg).
$$
for approximate skyline identification.  This resolves the sample complexity of approximate skyline identification up to constant factors for the entire range of the parameters $n, \eps, \delta$.

\subsection{Related Work} \label{sec:relatedwork}
The literature on multi-armed bandit problems is far too large to survey in its entirety, so we focus only on results that are closely related to ours.  The most similar work to ours is that of Auer et al.~\cite{Auer}, which considers a more general problem called \emph{Pareto-optimal arm identification}.  In this problem there are $n$ arms, each of which has a stochastic reward supported on $[0,1]^d$ with mean $\mu = (\mu_1,\dots,\mu_d) \in [0,1]^d$.  Although their parameterization is somewhat more general, they essentially prove an upper bound of $O(\frac{n}{\eps^2}\log\frac{nd}{\delta})$ to identify an $\eps$-Pareto-optimal set of arms with $1-\delta$.  The exact definition of $\eps$-Pareto optimal roughly corresponds to our definition of $\eps$-skyline, but the details are not crucial to this discussion.  Skyline identification can be cast as a special case of Pareto-optimal arm identification, but in this special case the sample complexity of their algorithm is no better than the na\"ive algorithm.  Indeed, in Auer et al.'s analysis, one of the first steps is to accurately estimate the payoff of all arms while paying a union bound in the sample complexity, and this union bound is precisely what our algorithm was designed to avoid.

There is also a related line of work on generalizations of the best-arm-identification problem in which the goal is to identify a \emph{set of $k$} arms with approximately maximal payoffs.  The first such work was by Kalyanakrishnan and Stone~\cite{KS10}, which introduced the $k$-best-arm-identification problem, and subsequent works~\cite{GGLB11, KTAS12, CGL16, CGLQW17} have proven tight upper and lower bounds, and also consider variants where the algorithm is only allowed to choose certain types of sets of $k$ arms.  Comparing our results to the upper and lower bounds of~\cite{KS10, KTAS12} reveals that the $\eps$-skyline dentification problem has the same sample complexity identifying an approximately optimal set of $\Theta(1/\eps)$ arms.  However, the two problems do not seem to be directly comparable, since an $\eps$-skyline is not necessarily a set of arms with approximately maximum reward and vice versa. 

The work by Kleinberg \cite{Kleinberg06} is also related---it develops fast-converging algorithms that find the best arm for every prefix of a countably infinite sequence of arms, under the regret minimization measure in contrast to the work in this paper which finds the best PAC arm in every prefix of a finite sequence.

\section{An Improved Algorithm for $\eps$-Skyline Identification}

\subsection{Overview of the Algorithm} \label{sec:algorithm_overview}
Our new algorithm that achieves lower sample complexity in the case of $\eps > 1/n$ can be viewed as a sample-efficient reduction from identifying an $\eps$-skyline to identifying an $\eps$-best arm.  For any set of arms $B$, with probability at least $1-\delta$, an algorithm of Even-Dar et al.~\cite{Even-Dar} finds an arm whose expected reward is within $\eps$ of the maximum among arms in $B$ using $O(\frac{|B|}{\eps^2}\log\frac{1}{\delta})$ samples.

Our reduction is roughly as follows.  Let $S$ be the set of arms our algorithm will identify, and initially set $S = \emptyset$.  Initially all of the $n$ arms are considered \emph{active}.    Our algorithm proceeds in a series of rounds, and we start by describing the first round.  Start by partitioning the arms into $t \approx \frac{20}{\eps}$ \emph{blocks} $B_1,\dots,B_{t}$, each of $\approx \frac{\eps n}{20}$ consecutive arms.\footnote{The constants in this informal description are somewhat arbitrary, and we have not optimized them.}  We then find an $\frac{\eps}{20}$-best arm $a_1,\dots,a_{t}$ for all of these blocks.  We then obtain an estimate $\hat{\mu}[a_1],\dots,\hat{\mu}[a_t]$ of the expected reward of each of these arms.  Using the algorithm of~\cite{Even-Dar} and paying a union bound over the blocks, we can do all of this using just $O(\frac{n}{\eps^2}\log \frac{1}{\eps \delta})$ samples.  

Now we scan the blocks from $1$ to $t$ (lower numbered arms are considered first) looking for arms to add to the skyline.  As we go through, each block $B_i$ will be either \emph{deactivated} or will remain (partially) active, depending on our estimate $\hat{\mu}[a_i]$ of the expected reward of the best arm in that block.  When we get to block $B_i$, suppose that $B_j$ is the last block that we kept active.  Then if $\hat{\mu}[a_i] > \hat{\mu}[a_j] + \frac{\eps}{5}$, we may need to include the arm $a_i$ in order to have an approximate skyline, so we add $a_i$ to the set $S$, and we leave block $B_i$ active.  (Our actual algorithm deactivates some of the arms in $B_i$ but this detail is not crucial for this informal summary.)  On the other hand, if $B_j$ is the last block that we left active and $\hat{\mu}[a_i] \leq \hat{\mu}[a_j] + \frac{\eps}{5}$, then no arm in $B_i$ dominates the best arm in $B_j$, which is already included in the set $S$.  Thus, we can safely deactivate all the arms in $B_i$.  

At the end of this process, some blocks will be deactivated and some blocks will remain at least partially active.  The key fact is that every time we keep a block $B_i$ active, its true expected reward must be at least $\frac{\eps}{10}$ larger than the last block we kept active, and therefore only $\frac{10}{\eps}$ blocks remain active, and these contain a total of at most $\frac{10}{\eps} \cdot \frac{\eps n}{20} = \frac{n}{2}$ arms.  Thus, we have made significant progress in the sense that we have eliminated a constant fraction of the arms from consideration while still leaving enough arms to guarantee that we can form an approximate skyline.  Now we can recurse on this set remaining $\frac{n}{2}$ arms by further dividing each of the active blocks into $\approx \frac{20}{\eps}$ sub-blocks.  We repeat essentially the same process on these sub-blocks.  See Figure~\ref{picture} for a diagram of how the algorithm progresses.

In the rest of this section we describe and analyze our algorithm, and thereby prove Theorem~\ref{thm:main}.  Before describing the algorithm in detail, we will need to introduce a few useful subroutines, and some helpful terminology.
\subsection{Useful Algorithmic Subroutines}
The first primitive simply captures the fact that we can estimate the mean of any given arm with high confidence by drawing a sufficient number of samples.  We will denote this procedure by the subroutine \estmean.  The proof of Lemma~\ref{lem:estmean} is a standard application of Hoeffding's Inequality, and is omitted.
\begin{lem}[\estmean] \label{lem:estmean}
There is an algorithm \estmean~that takes a tuple $(A[i],\eps,\delta)$ as input, draws $O(\frac{1}{\eps^2}\log \frac{1}{\delta})$ samples from $A[i]$ and outputs an estimate $\hm[i] \in[0,1]$ such that, with probability at least $1-\delta$, $|\mu[i] - \hm[i]| \leq \eps.$
\end{lem}

The second subroutine is an algorithm of Even-Dar et al.~\cite{Even-Dar} that identifies an arm with approximately maximum payoff among any set of arms $A$, using sample complexity that is just linear in $|A|$.  First, we introduce some useful terminology for discussing multi-armed bandit problems.
\begin{defn}[$\eps$-Better and $\eps$-Best]
\label{def:bestarm}
Arm $A[i]$ is \emph{$\eps$-better} than arm $A[j]$ if the expected payoff of $A[i]$ is at least within $\eps$ of that of $A[j]$.  That is, $\mu[i] \geq \mu[j] - \eps$.  For a set of arms $A$, $A[i] \in A$ is an \emph{$\eps$-best arm for $A$} if $A[i]$ is $\eps$-better than all arms $A[j] \in A$.  That is $\forall A[j]\in A~~\mu[i] \geq \mu[j] - \eps$.
\end{defn}

\begin{thm}[\findbest] \label{thm:findbest}
There is an algorithm \findbest\ that takes a tuple $(A', \eps, \delta)$, where $A' \subseteq A$ is a set of arms and $\eps,\delta$ are parameters, draws $O(\frac{|A'|}{\eps^2}\log \frac{1}{\delta})$ samples from the arms in $A'$, and outputs an arm $A[i] \in A'$ such that, with probability at least $1-\delta$, $A[i]$ is an $\eps$-best arm for $A'$.
\end{thm}

\subsection{The Algorithm}
We are now ready to describe our algorithm.  The bulk of the work is done in \Cref{algo:main}, which identifies an approximately optimal set of arms.  We also define an auxiliary procedure\Cref{algo:truncate} that pares down the set of arms to have size $O(1/\eps)$, which will establish the final guarantee of Theorem~\ref{thm:main}.  We remark that, by Definition~\ref{def:runningmax}, the first arm is always part of an $\eps$-skyline, thus it will be more notationally convenient in our algorithm if we consider a set of $n+1$ arms labeled $A[0], A[1],\dots,A[n]$.

\begin{algorithm}
\caption{}
\label{algo:main}
{\bf Input:} Sample access for arms $A = \{A[0],A[1],\dots,A[n]\}$ and error parameters $\eps,\delta$.\\
{\bf Output:} An $\eps$-skyline of $S \subseteq A$.
\begin{algorithmic}[1]
\STATE $S\leftarrow \{0\}$\COMMENT{A set to hold the $\eps$-skyline}
\STATE $\hm[0]\leftarrow$ \estmean$(A[0],\eps/12,\delta/2)$
\STATE $M\leftarrow \{\hm[0]\}$\COMMENT{A set to hold measurements of arms in $S$}
\STATE $B_1\leftarrow \{[1,2,\ldots n]\}$\COMMENT{The set of blocks for the 1st level}
\FOR[Outer loop over levels]{$\ell = 1,2,\dots$}
	\STATE $b_\ell \leftarrow |B_\ell|$
    \IF[Stop when no blocks are left]{$b_\ell=0$}
        \STATE \textbf{break}
	\ENDIF
    
	\STATE $\delta_\ell \leftarrow \frac{\delta}{2^{\ell+1}}$
    \STATE $B_{\ell+1}\leftarrow \emptyset$\COMMENT{The set of blocks for the next level (initially empty)}
    
    \FORALL[Inner loop over blocks in this level]{$m: = 1,\dots, b_\ell$}
    	\STATE $\delta_{\ell,m} \leftarrow \delta_\ell / b_\ell$
        
    	\STATE $[i_{\ell,m}\ldots j_{\ell,m}]\leftarrow B_{\ell,m}$\COMMENT{$m^{th}$ block of $\ell^{th}$ level}
        \STATE $prev\leftarrow$ largest member of $S$ to the left of $i_{\ell,m}$
        \STATE $L_{\ell,m}\leftarrow (\eps/2)+\hm[prev]$ \COMMENT{Retrieve $\hm[prev]$ from $M$}
		\STATE $k_{\ell,m}\leftarrow$ \findbest $(B_{\ell,m}, \eps/12, \delta_{\ell,m}/2)$
    	\STATE $\hm[k_{\ell,m}]\leftarrow$ \estmean$(A[k_{\ell,m}], \eps/12,\delta_{\ell,m}/2)$
        
    	\IF{$L_{\ell,m} +\eps/4 > \hm[k_{\ell,m}]$}
        	\STATE \textbf{continue}\COMMENT{Deactivate this block (red in Fig. 1)}
        \ENDIF
        
        \STATE $S\leftarrow S\cup k_{\ell,m}$ \COMMENT{Otherwise, store index (blue in Fig. 1)}
        \STATE $M\leftarrow M\cup \hm[k_{\ell,m}]$ \COMMENT{Store measurement}
        \STATE $U_{\ell,m}\leftarrow \hm[k_{\ell,m}] + \eps/6$ \COMMENT{Upper bound on arm quality in block}
        
        \STATE $b_{\ell,m}\leftarrow (4/\eps)(U_{\ell,m}-L_{\ell,m})$ \COMMENT{Target no. of blocks to produce}
        \STATE Break $[i_{\ell,m}\ldots k_{\ell,m}-1]$ into blocks of size $\max(1,\lfloor (k_{\ell,m}-i_{\ell,m})/b_{\ell,m} \rfloor)$
        \STATE $B_{\ell+1}\leftarrow B_{\ell+1} \cup$ those new blocks
	\ENDFOR
\ENDFOR
\RETURN{$S$}
\end{algorithmic}

\end{algorithm}

\begin{algorithm}
\caption{}
\label{algo:truncate}
{\bf Input:} The $\eps$-skyline $S$ output by \Cref{algo:main} and the set $M$ from its execution \\
{\bf Output:} An $\eps$-skyline $S$ of size $O(1/\eps)$
\begin{algorithmic}[1]
\STATE $s\leftarrow 0$
\WHILE{$s\neq \max S$}
	\STATE Let $s'$ denote the next member of $S$
    \STATE Fetch $\hm[s],\hm[s']$ from $M$\COMMENT{These estimates were already obtained by \Cref{algo:main}}
	\IF{$\hm[s] + 3\eps/4 > \hm[s']$}
    	\STATE Remove $s'$ from $S$
    \ELSE
    	\STATE $s\leftarrow s'$
	\ENDIF
\ENDWHILE
\RETURN{$S$}
\end{algorithmic}
\end{algorithm}

\begin{figure}
\begin{center}
\includegraphics[width=.75\textwidth]{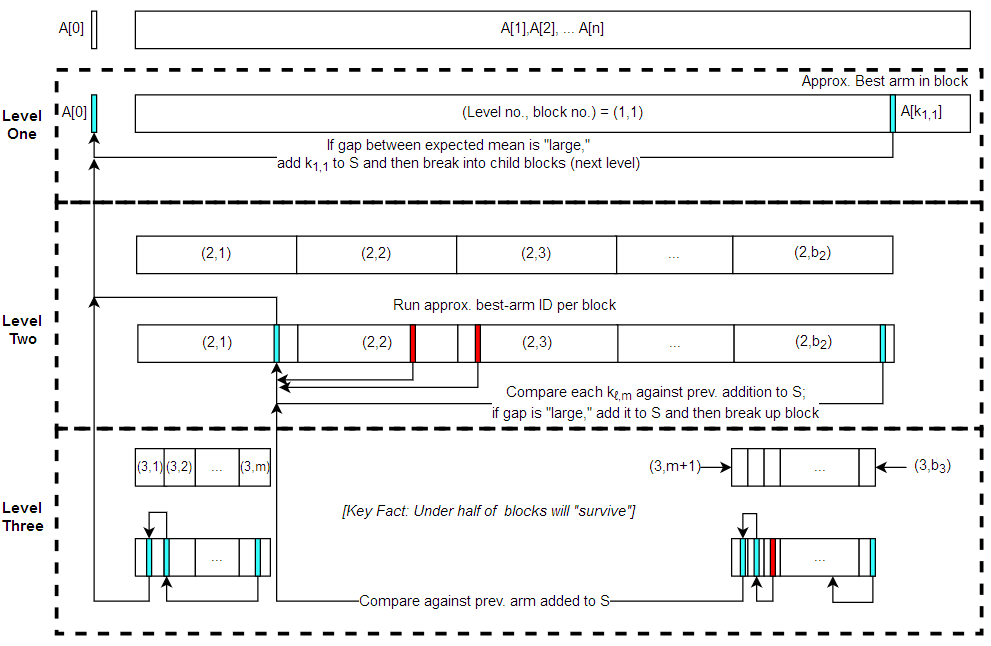}
\caption{Sample execution of \Cref{algo:main}. Colored slits indicate best arms of blocks. Blue slits are added to $S$, red are not. Arrows indicate comparisons between measurements of arms. Red blocks are deactivated while portions of blue blocks are broken into sub-blocks.}
\label{picture}
\end{center}
\end{figure}

\subsection{Analysis}

To aid in the analysis, we recall some notation for the algorithm and its execution.  The algorithm proceeds in levels and in each level the arms will be partitioned into blocks.  We use $\ell$ to index levels and $m$ to index blocks, so that $B_{\ell, m}$ is the $m$-th block in the $\ell$-th level (which is given the shorthand $(\ell, m)$ in Figure \ref{picture}).  We use $n_{\ell, m}$ to denote the number of arms in block $B_{\ell,m}$ and $n_{\ell}$ to denote the number of arms in level $\ell$.  For each block, $B_{\ell, m}$ the algorithm $\findbest$ outputs an arm $k_{\ell,m}$ that is purported to be an approximate best arm in $B_{\ell,m}$, and then we call $\estmean$ to obtain an estimate $\hat{\mu}[k_{\ell,m}] \approx \mu[k_{\ell,m}]$.  We let $\mu^*_{\ell,m}$ denote the exact mean of the exact best arm in $B_{\ell, m}$.

First, we will show that with high probability, all calls to \estmean~and \findbest~return ``correct'' outputs.  This will allow us to proceed with all of our analysis as if these two algorithms succeed with probability $1$.  To do so, we first define key events.

Consider round $\ell,m$ in \Cref{algo:main}. We shall say the \findbest\ call at that round succeeds if it identifies an arm $A[k_{\ell,m}]$ that is close-to-best in the block:
\begin{equation*}
\mu^*_{\ell,m}-\mu[k_{\ell,m}]<\eps/12 \stepcounter{equation}\tag{\theequation}\label{eq:goodident}
\end{equation*}
and, similarly, \estmean\ succeeds if the measurement of $A[k_{\ell,m}]$ has low error:
\begin{equation*}
|\mu[k_{\ell,m}]-\hm[k_{\ell,m}]|<\eps/12 \stepcounter{equation}\tag{\theequation}\label{eq:goodmsmnt}
\end{equation*}

Let $E_{\ell,m}$ be the event that both calls to \estmean\ and \findbest\ succeed in round $(\ell,m)$. Likewise, let $E_\ell$ be the event that all calls to \estmean\ and \findbest\ succeed in level $\ell$.  Let $E$ be the success event for the whole of \Cref{algo:main}: $E$ is the event that line 1's \estmean\ succeeds and $E_{\ell,m}$ occurs for all rounds of the nested loop.  We prove that all calls are successful with high probability.
\begin{clm}\label{clm:whp}
The probability that $E$ occurs is at least $1-\delta$.
\end{clm}

\begin{proof}

In level $\ell$, we make $2b_\ell$ calls to \findbest\ and \estmean, each with error-probability parameter $\delta_{\ell,m}/2=\delta_\ell/(2b_\ell)$. From \Cref{thm:findbest} and \Cref{lem:estmean}, the probability that any one of the calls does not succeed is $\prob(\neg E_\ell)\leq \delta_\ell$.  Since $\delta_\ell = \frac{\delta}{2^{\ell+1}}$, the failure probability over all levels is $\leq \delta/2$.  The probability that \estmean\ on line 1 fails is also $\leq \delta/2$.  Having covered all executions of \estmean\ and \findbest, we conclude that the probability of failure is $\prob(\neg E)\leq \delta$.
\end{proof}

We now make two claims that state that if $E$ holds, then \Cref{algo:main}~outputs an $\eps$-skyline of $A$ and has the desired sample complexity.  The proofs of these claims are in Sections~\ref{sec:correctnessproof} and~\ref{sec:complexityproof}, respectively.
\begin{clm}\label{clm:correctness}
If $E$ occurs, then \Cref{algo:main}\ outputs an $\eps$-skyline of $A$.
\end{clm}

\begin{clm}\label{clm:complexity}
If $E$ occurs, then \Cref{algo:main}\ takes $O(\frac{n}{\eps^2}\log \frac{1}{\eps\delta})$ samples.
\end{clm}

\subsubsection{Correctness of \Cref{algo:main}\ (\Cref{clm:correctness})} \label{sec:correctnessproof}

Recall from \Cref{def:runningmax} that $S$, the output of \Cref{algo:main}, is an $\eps$-skyline of $A$ if and only if the following two hold:
\begin{enumerate}
\item For any arm index $t\notin S$, the largest $s\in S$ such that $s < t$ satisfies $\mu[s] \geq \mu[t]-\eps$
\item Any $s\in S$ is the index to an $\eps$-best arm of $\{A[0],A[1],\ldots, A[s]\}$
\end{enumerate}
We will prove each of these two points separately.  Throughout this section we will condition on the event $E$, that is we will assume throughout that~\eqref{eq:goodident} and \eqref{eq:goodmsmnt} hold.

\begin{lem}
\label{lem:nofalseneg}
If $E$ occurs, then for any $t\notin S$, the largest $s\in S$ such that $s < t$ satisfies $\mu[s] \geq \mu[t]-\eps$
\end{lem}

\begin{proof}
Because $t\notin S$, the arm $A[t]$ was deactivated at some level $\ell$ in \Cref{algo:main}\ when it was a member of $m$-th block of that level; let $B_{\ell,m}$ denote the block. Let $k_{\ell,m}$ denote the index identified by \findbest.  There are two ways that $A[t]$ could have been deactivated.

\medskip

\textbf{Case 1:} $A[t]$ was deactivated when $k_{\ell,m}$ was added to $S$.  In this case, \Cref{algo:main} deactivates \emph{only} arms to the right of $k_{\ell,m}$ in $B_{\ell,m}$ (see line 24). Because $A[t]$ was deactivated, this means $k_{\ell,m}<t$.  We also know \Cref{algo:main} deactivates \emph{all} arms to the right of $k_{\ell,m}$ in $B_{\ell,m}$: no indices between $k_{\ell,m}$ and $t$ will ever be added to $S$. Ergo, $k_{\ell,m}$ is the largest member of $S$ to the left of $t$.  From \Cref{thm:findbest} and $E$, $k_{\ell,m}$ is an $\eps/12$ best arm of $B_{\ell,m}$. Thus $\mu[k_{\ell,m}] \geq \mu[t]-\eps/12$.  In sum, $k_{\ell,m}$ is the largest member of $S$ to the left of $t$ and $\mu[k_{\ell,m}] \geq \mu[t]-\eps/12$. \Cref{lem:nofalseneg} holds in this case.

\medskip

\textbf{Case 2:} $A[t]$ was deactivated along with all other members of $B_{\ell,m}$. In this case, we shall first argue \textbf{(2a)} that $prev$, defined in \Cref{algo:main} as the largest $s\in S$ such that $s<i_{\ell,m}$ \emph{at level $\ell$}, is the largest $s\in S$ such that $s<t$ \emph{at the end of the algorithm}. Then we shall argue \textbf{(2b)} that the inequality $\mu[prev] \geq \mu[t]-\eps$ holds.

\textbf{(2a).} Because the entirety of $B_{\ell,m}$ is deactivated, no indices between $i_{\ell,m}$ and $t$ will ever be added to $S$. Ergo, the largest $s\in S$ such that $s<i_{\ell,m}$ (at the end of \Cref{algo:main}) must be the same as the largest $s\in S$ such that $s<t$ (at the end of \Cref{algo:main}).

By definition, no index between $prev$ and $i_{\ell,m}$ is in $S$ at level $\ell$. Observe that all arms between $A[prev]$ and $A[i_{\ell,m}]$ are deactivated by the end of level $\ell$: were it not so, some block would have contributed an arm index to $S$. Ergo, no indices between $prev$ and $i_{\ell,m}$ will ever be added to $S$.

\textbf{(2b).} Having shown that $prev$ is the largest member of $S$ to the left of $t$ when \Cref{algo:main} ends, now we show that $\mu[prev]\geq \mu[t]-\eps$. Because $k_{\ell,m}\notin S$, we have
\begin{align*}
L_{\ell,m} +\eps/4=\hm[prev]+3\eps/4 &> \hm[k_{\ell,m}]
\shortintertext{otherwise we would have added the index by line 18. By assuming $E$, we can use the fact that measurements are accurate \eqref{eq:goodmsmnt}:}
\mu[prev] &> \mu[k_{\ell,m}]-11\eps/12 \stepcounter{equation}\tag{\theequation}\label{eq:stepdown}
\shortintertext{By assuming $E$, we can use \eqref{eq:goodident}, the fact that $A[k_{\ell,m}]$ has mean close to the best arm of $B_{\ell,m}$, which has mean $\mu^*_{\ell,m}$:}
\mu[prev] &> \mu^*_{\ell,m} - \eps
\end{align*}
By definition, $\mu^*_{\ell,m}$ is the maximum mean in $B_{\ell,m}$, so $\mu[prev] > \mu_t-\eps$.  In sum, $prev$ is the largest member of $S$ to the left of $t$ and $\mu[prev]\geq \mu[t]-\eps$. \Cref{lem:nofalseneg} holds in this case.

\medskip
Combining the two cases completes the proof of \Cref{lem:nofalseneg}.
\end{proof}

Now we shall prove the second part of \Cref{def:runningmax}---that the members of $S$ are approximate best arms for their prefixes.
\begin{lem}
\label{lem:nofalsepos}
If $E$ occurs, then any $s\in S$ is the index to an $\eps$-best arm of $\{A[0]\ldots A[s]\}$
\end{lem}
\begin{proof}
Label the arms in $S$ by $S[1],S[2],\dots,$ sorted by their position in the original list of arms $A$, so that $S[1] < S[2] < \dots$.  Note that this ordering does not correspond to the order that the arms were added to $S$ in the execution of Algorithm \Cref{algo:main}. In keeping with the figures, we will visualize the arms in $A$ and the elements of $S$ from left to right where the lower-numbered arms are left and higher-numbered arms are right.   We prove the lemma by induction on the list of arms $S[1],S[2],\dots$ from left to right.

\medskip
The base case is $r=1$, where $S[1]=0$. Clearly, $A[0]$ is the best arm of the trivial prefix $\{0\}$.

\medskip
Now assume the inductive hypothesis that for all $r'\in[1,r-1]$, $S[r']$ is an index of an $\eps$-best arm of the prefix $\{0,1,\ldots S[r']\}$. We will prove that $S[r]$ is an index to an $\eps$-best arm of $\{0,1,\ldots S[r]\}$; because every addition to $S$ is an output of \findbest, $S[r]$ is an $\eps$-best arm $k_{\ell,m}$ from some $B_{\ell, m}$, and was added to $S$ from some $B_{\ell,m}$ because it did not satisfy the condition on line 18, 
$$
L_{\ell,m} +\eps/4= \hm[prev]+3\eps/4 \leq \hm[k_{\ell,m}]
$$
where $prev$ is the nearest member of $S$ to the left of $B_{\ell,m}$ in round $(\ell,m)$. Because we assumed $E$, all calls to \estmean\ succeeded; using \eqref{eq:goodmsmnt},
\begin{equation}
\mu[prev]+7\eps/12 \leq \mu[k_{\ell,m}] \stepcounter{equation} \label{eq:stepup}
\end{equation}
We will use the above to show that $k_{\ell,m}$ indexes an $\eps$-best arm of $A[0,1,\ldots k_{\ell,m}]$. To do so, we decompose the interval into subintervals and prove that $A[k_{\ell,m}]$ is $\eps$-better than all arms in each. They are labeled as intervals (a,b,c,d) in \Cref{proofpicture}.

Let $B_{\ell,m}$ be the block which contained $k_{\ell,m}$ when $k_{\ell,m}$ was added to $S$. It spans $i_{\ell,m}$ and $j_{\ell,m}$. Let $B_{prev}$ be the block which contained $prev$ when $prev$ was added to $S$. It spans $i_{prev}$ and $j_{prev}$.  The subintervals we shall analyze are $[0,prev],[i_{prev},j_{prev}],(j_{prev},i_{\ell,m}),[i_{\ell,m},k_{\ell,m}]$; in \Cref{proofpicture}, they are labeled a-d. Note that they are not all disjoint but this does not affect the conclusion.  For each labeled interval, the corresponding entry in the list below proves that $A[k_{\ell,m}]$ is $\eps$-better than all arms in the interval.

\begin{figure}
\begin{center}
\includegraphics[width=0.75\textwidth]{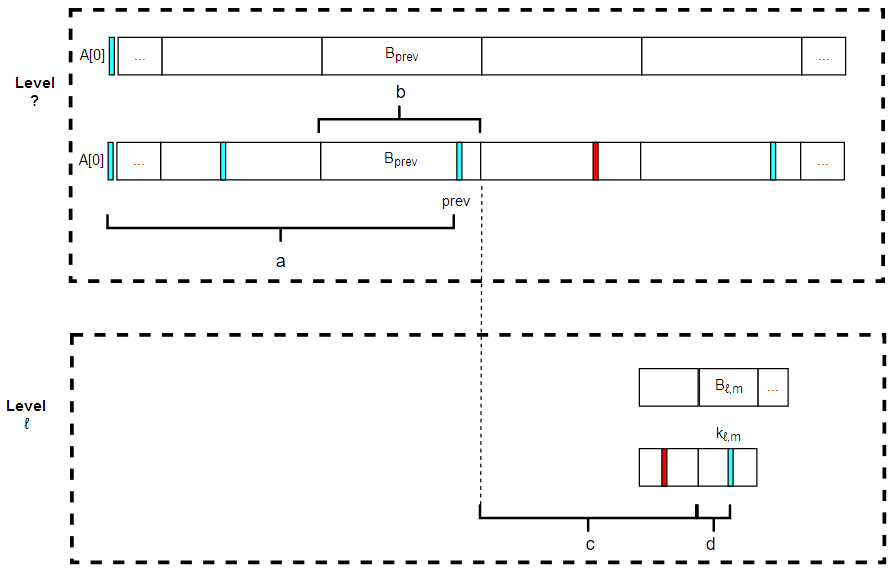}
\caption{Decomposition of the prefix $[0,k_{\ell,m}]$ into subintervals in terms of $prev$, the closest member of $S$ to the left.}
\label{proofpicture}
\end{center}
\end{figure}

\begin{enumerate}
\item[(a)] By inductive assumption, $A[prev]$ is an $\eps$-best arm of $[0,prev]$. By \eqref{eq:stepup}, the mean of $A[k_{\ell,m}]$ can only be larger..

\item[(b)] Because every member of $S$ is an $\eps$-best arm of the block that last contained it, $A[prev]$ is an $\eps$-best arm of $B_{prev}=[i_{prev},j_{prev}]$. By \eqref{eq:stepup}, the mean of $A[k_{\ell,m}]$ can only be larger.

\item[(c)] Here we show that $A[k_{\ell,m}]$ is $\eps$-better than all of $(j_{prev},i_{\ell,m})$. When $k_{\ell,m}$ was added to $S$, no arm in that interval was in $S$ because $prev$ is defined to be the closest member of $S$ to the left. These arms are deactivated; by \Cref{lem:nofalseneg}, $A[prev]$ is $\eps$-better than all of them. By \eqref{eq:stepup}, the mean of $A[k_{\ell,m}]$ can only be larger.

\item[(d)] $A[k_{\ell,m}]$ is an $\eps$-best arm of $B_{\ell,m}$ which implies it is an $\eps$-best arm of the subset $[i_{\ell,m},k_{\ell,m}]$.
\end{enumerate}

Because the union of the above intervals is $[0,k_{\ell,m}]$, $A[k_{\ell,m}]$ is an $\eps$-best arm of that prefix.  Thus, $S[r]=k_{\ell,m}$ indexes an $\eps$-best arm of $[0,S[r]]$.  The lemma now follows by induction.
\end{proof}

\subsubsection{Efficiency of \Cref{algo:main} (\Cref{clm:complexity})}
\label{sec:complexityproof}

We now turn to showing that, if~\eqref{eq:goodident} and \eqref{eq:goodmsmnt} hold, then \Cref{algo:main} draws $O(\frac{n}{\eps^2} \log \frac{1}{\eps \delta})$ samples.  First we will outline the proof while introducing some technical lemmas, then we will conclude by proving those lemmas.

\begin{proof}[Proof of \Cref{clm:complexity}]
Observe that \Cref{algo:main} only draws samples via the subroutines \estmean\ and \findbest. Measuring $A[0]$ on Line 1 of \Cref{algo:main} is a call to \estmean with error-parameters $(\eps/12,\delta/2)$, so it requires just $O\left(\frac{1}{\eps^2}\log\frac{1}{\delta}\right)$ samples by \Cref{lem:estmean}.  It remains to show that the total number of samples taken by \estmean\ and \findbest\ in the main loop is bounded by $O(\frac{n}{\eps^2}\log \frac{1}{\eps\delta})$.

Recall $n_{\ell,m}$ is the number of arms in block $B_{\ell,m}$ and $n_\ell$ is the the number arms in level $\ell$.  Let $C_\ell$ denote the sample complexity at level $\ell$ and let $C_{\ell,m}$ denote the sample complexity incurred by block $B_{\ell,m}$ so that
\[
C_\ell = \sum_{m=1}^{b_\ell} C_{\ell,m}
\]

When considering block $(\ell,m)$, \estmean\ and \findbest\ are called with parameters $(\eps/12,\delta_{\ell,m}/2)$. By \Cref{lem:estmean} and \Cref{thm:findbest}, the subroutines together take $O((n_{\ell,m}/\eps^2)\log (1/\delta_{\ell,m}))$ samples.  Recalling that $\delta_{\ell,m} = \delta_{\ell} / b_{\ell}$ where $\delta_{\ell} = \delta / 2^{\ell+1}$ and where $b_{\ell}$ is the number of blocks in the level, we can write

\begin{align*}
C_\ell 
&={} O\left(\sum_{m=1}^{b_\ell}\frac{n_{\ell,m}}{\eps^2}\log\frac{1}{\delta_{\ell,m}} \right) = O\left(\frac{1}{\eps^2} \sum_{m=1}^{b_\ell} n_{\ell,m} \log\frac{1}{\delta_{\ell,m}} \right) \\
&={} O\left(\frac{1}{\eps^2}\log\frac{b_\ell}{\delta_\ell}\sum_{m=1}^{b_\ell}n_{\ell,m} \right) = O\left(\frac{n_\ell}{\eps^2}\log\frac{b_\ell}{\delta_\ell}\right) = O\left(\frac{n_\ell}{\eps^2} \left(\log\frac{b_\ell}{\delta}+\log 2^{\ell+1} \right) \right)\\ 
&={} O\left(\frac{n_\ell}{\eps^2} \left( \ell +\log\frac{b_\ell}{\delta} \right)\right)
\end{align*}

Note that $b_{\ell}$ is increasing with $\ell$.  To finish the analysis, we will show that $n_{\ell}$ shrinks exponentially with $\ell$ sufficient speed that the overall sample complexity at each level is also decaying exponentially.

\medskip
First, we show that the number of blocks $b_{\ell}$ can increase at most exponentially in $\ell$.
\begin{lem}\label{lem:boundblocks}
If $E$ occurs, then at every level $\ell$, there are $b_\ell = O((\frac{10}{3})^\ell \cdot \frac{1}{\eps})$ blocks
\end{lem}
\noindent The proof will be presented shortly. From this lemma, we have 
\[
C_\ell = O\left(\frac{n_\ell}{\eps^2} \left( \ell +\log\frac{1}{\eps\delta} \right) \right) ~~~~~ (*)
\]
Let $\tau$ denote the number of levels $\ell$ considered by the algorithm before terminating, let $C$ denote the sample complexity of the main loop. Naturally,
\begin{align*}
C &= \sum_{\ell=1}^{\tau} C_\ell = O\left( \frac{1}{\eps^2} \sum_{\ell=1}^{\tau} n_\ell \left( \ell +\log\frac{1}{\eps\delta} \right) \right)\\
&= O\left( \frac{1}{\eps^2} \sum_{\ell=1}^{\infty} n_\ell \left( \ell + \log\frac{1}{\eps\delta} \right) \right)~~~~~(**)
\end{align*}
where we use $(*)$ to expand $C_\ell$.

\medskip
Finally, we show that the number of arms $n_{\ell}$ decreases exponentially in $\ell$
\begin{lem} \label{lem:boundarms}
If $E$ occurs, the number of active arms in level $\ell$ is $n_{\ell} \leq \half \cdot n_{\ell-2}$ for all $\ell>2$.
\end{lem}

The proof will be presented shortly. From this lemma, we have
\begin{align*}
C = (**) &= O\left(\frac{n}{\eps^2} \cdot \sum_{\ell=1}^{\infty} \left(\half\right)^\ell \left( \ell + \log\frac{1}{\eps\delta} \right) \right)\\
&= O\left(\frac{n}{\eps^2}\log\frac{1}{\eps\delta}\right)
\end{align*}
\end{proof}

Now that we have outlined the structure of the proof, we complete the analysis by proving \Cref{lem:boundblocks} and \Cref{lem:boundarms}.

\begin{proof}[Proof of \Cref{lem:boundblocks}]
Let $P_\ell$ denote the set of blocks in level $\ell$ that are ``parents'' of ``child'' blocks, by which we mean those blocks of $\ell$ that are divided into sublocks for $\ell+1$.  Naturally, $P_\ell \subseteq B_\ell$ and thus $|P_\ell|\leq |B_\ell| = b_\ell$.

If $B_{\ell,m} \in P_\ell$, then line 24 bounds the number of children by $$b_{\ell,m}:=\frac{4}{\eps}(U_{\ell,m}-L_{\ell,m})=(\hm[k_{\ell,m}]-\hm[prev]-\eps/3)$$ where $prev$ denotes the largest member of $S$ smaller than $k_{\ell,m}$ at round $(\ell,m)$.  Thus the number of blocks in $\ell+1$ is bounded by
\[
b_{\ell+1} \leq \sum_{B_{\ell,m}\in P_\ell} b_{\ell,m} = \frac{4}{\eps} \left( -\frac{\eps}{3} |P_\ell|+ \sum_{B_{\ell,m}\in P_\ell} \hm[k_{\ell,m}]-\hm[prev_{\ell,m}] \right) ~~~ (***)
\]
where we insert subscript to $prev$ to disambiguate between rounds. 

We will upper bound the term in the summation by lower bounding $\hm[prev_{\ell,m}]$.  Consider two members of $P_\ell$, $B_{\ell,m'}$ and $B_{\ell,m}$, such that $m'<m$ and there is no member of $P_\ell$ between them. By the manner in which $prev_{\ell,m}$ is defined, $k_{\ell,m'} \leq prev_{\ell,m}$.  Recall \Cref{lem:nofalsepos}, which states that $\forall s\in S\ A[s]$ is an $\eps$-best arm of the prefix $[s]$. This means $\mu[prev_{\ell,m}] > \mu[k_{\ell,m'}]-\eps$.  Recall that \estmean\ is given error parameter $\eps/12$. Because we assumed $E$, \eqref{eq:goodmsmnt} gives $\hm[prev_{\ell,m}] > \hm[k_{\ell,m'}]-7\eps/6$.  Therefore $\hm[k_{\ell,m}]-\hm[prev_{\ell,m}]<\hm[k_{\ell,m}]-\hm[k_{\ell,m'}]+7\eps/6$. This lets us upper bound $(***)$ by
\begin{align*}
b_{\ell+1} &\leq \frac{4}{\eps} \left( \frac{5 \eps}{6} |P_\ell| + \sum_{B_{\ell,m}\in P_\ell} \hm[k_{\ell,m}]-\hm[k_{\ell,m'}] \right)
\shortintertext{Measurements lie within $[0,1]$ so the summation telescopes nicely}
&< \frac{4}{\eps} \left( \frac{5\eps}{6}|P_\ell| + 1 \right) < \frac{10}{3} b_\ell+ \frac{4}{\eps}
\end{align*}

Since $b_1=1$, we can solve the recurrence to obtain
\begin{align*}
b_\ell &\leq \left(\frac{10}{3}\right)^\ell + \frac{4}{\eps} \sum_{\ell'=0}^{\ell-1} \left(\frac{10}{3}\right)^{\ell'}\\
&< \left(\frac{10}{3}\right)^\ell \left(1+ \frac{4}{\eps}\right) = O\left(\left(\frac{10}{3}\right)^\ell\cdot \frac{1}{\eps}\right)
\end{align*}
\end{proof}

Finally, we prove \Cref{lem:boundarms}, stating that the number of active arms in a level $\ell$ decays exponentially with $\ell$.  To do so we need two intermediate lemmas concerning the children of a given block. \Cref{lem:boundcont} states that, for every block $B_{\ell,m}$, there is a bound on the number of arms its children can ``contribute'' to $S$.  \Cref{lem:boundarm-helper} relates \Cref{lem:boundcont} to $b_{\ell,m}$, which is the number of children of $B_{\ell,m}$. Because each child contributes at most one arm to $S$ (specifically, an arm identified by \findbest), if the number of child blocks is twice the maximum number of contributions (which we ensure in the algorithm), at least half of the child blocks do not make a contribution and are deactivated.

Once we have this guarantee for every block in a level $\ell$, if child blocks are of roughly equal size (which we ensure in the algorithm), it follows that at least half of the active arms in level $\ell$ are deactivated before level $\ell+2$.

\begin{lem}\label{lem:boundcont}
Assuming $E$ occurs, if block $B_{\ell,m}$ contributes $k_{\ell,m}$ to $S$, the number of contributions made by its children is bounded by $$\frac{2}{\eps}(U_{\ell,m}-L_{\ell,m})= \frac{2}{\eps} (\hm[k_{\ell,m}]-\hm[prev]-\eps/3).$$
\end{lem}
\begin{proof}[Proof of Lemma~\ref{lem:boundcont}]
We will show that the mean of any arm in $B_{\ell,m}$ resides within an interval $(L_{\ell,m},U_{\ell,m})$ and that the $(r+1)$-st-from-the-left contribution to $S$ by the children implies a ``large step'' from the $r$-th. It follows that there can only be so many steps.

First, observe that the mean of the leftmost child's contribution is greater than $\hm[prev]+\eps/2=L_{\ell,m}$.  Specifically, consider $k_1$, the leftmost contribution to $S$ by a child block. For an addition to $S$ to occur, the inequality of line 18 is false: $\hm[k_1]$ is larger than the the measured mean of the closest arm to the left. By assumption, this is $prev$. \eqref{eq:stepup} tells us that $\mu[k_1] \geq \mu[prev] + 7\eps/12$. From \eqref{eq:goodmsmnt} we know $\mu[prev]\geq \hm[prev]-\eps/12$ so $\mu[k_1] \geq \hm[prev]+\eps/2$.

Second, observe that the mean of any child's contribution is at most  $\leq \hm[k_{\ell,m}]+\eps/6=U_{\ell,m}$.  Specifically, by assuming $E$, we have from \eqref{eq:goodmsmnt} that $\mu[k_{\ell,m}] \leq \hm[k_{\ell,m}]+\eps/12$ and from \eqref{eq:goodident} that $|\mu[k_{\ell,m}]-\mu^*_{\ell,m}|<\eps/12$. Combining these inequalities, we have $\mu^*_{\ell,m} \leq \mu[k_{\ell,m}]+\eps/12 \leq \hm[k_{\ell,m}]+\eps/6$.

From this, we can show that every contribution of level $\ell+1$ has mean $>\eps/2$ more than the previous such contribution.  Consider $B_{\ell+1,p}$, the child block of $B_{\ell,m}$ that makes the $r^{th}$ contribution in level $\ell+1$. Let $B_{\ell+1,q}$ denote the block that makes the $(r+1)^{th}$.  At the time of adding $k_{\ell+1,q}$ to $S$, $k_{\ell+1,p}$ is the closest member of $S$ to its left. Because $k_{\ell+1,q}$ was added to $S$, the inequality of line 18 is false: from \eqref{eq:stepup}, $\mu[k_{\ell+1,q}] > \mu[k_{\ell+1,p}] + \eps/2$ where $k_{\ell+1,p}$ stands in for $prev$.

Combining these three facts immediately implies the lemma.
\end{proof}

\begin{lem}\label{lem:boundarm-helper}
Assuming $E$ occurs, if $B_{\ell,m}$ contributes some $k_{\ell,m}$ to $S$ in level $\ell$, then at most half of the arms in $B_{\ell,m}$ are active in level $\ell+2$.
\end{lem}
\begin{proof}[Proof of Lemma~\ref{lem:boundarm-helper}]
There are two cases, depending on the size of $B_{\ell,m}$ relative to $b_{\ell,m}$, an upper bound on the number of child blocks that will be made.

\medskip
\textbf{Case 1:} $k_{\ell,m}-i_{\ell,m} \geq b_{\ell,m}$.  Because there are enough arms, the algorithm breaks $B_{\ell,m}$ into exactly $b_{\ell,m}=(4/\eps)(U_{\ell,m}-L_{\ell,m})$ child blocks (on line 24).  \Cref{lem:boundcont} bounds the number of contributions in level $\ell+1$ by $(2/\eps)(U_{\ell,m}-L_{\ell,m})$. Thus the number of contributions is at most half of $b_{\ell,m}$.  If a block does not add an arm to $S$, all arms in that block are deactivated. Because at most half of the child blocks contribute to $S$ and the child blocks are of (approximately) equal size, at most half of the arms in $B_{\ell,m}$ are active after level $\ell+1$.

\medskip
\textbf{Case 2:} $k_{\ell,m}-i_{\ell,m} < b_{\ell,m}$.  Line 24 breaks $B_{\ell,m}$ into single-arm blocks. If one of those blocks does not contribute to $S$, then the block is deactivated. If it does contribute to $S$, line 24 cannot form child blocks.  Thus no arm in $B_{\ell,m}$ is active after level $\ell+1$.
\end{proof}

Given \Cref{lem:boundarm-helper}, it is clear that the number of active arms of a level decays exponentially, which is the assertion of \Cref{lem:boundarms}.

\subsubsection{Analysis of \Cref{algo:truncate}}
\begin{clm}
Let $S$ be the output of \Cref{algo:main}\ on input $(A,\eps,\delta)$. If $E$ occurs, the output of \Cref{algo:truncate} on $S$ is an $\eps$-skyline of size $O(1/\eps)$.
\end{clm}
Note that \Cref{algo:truncate} is deterministic, and we only rely on the condition $E$ to establish that the input $S$ is an $\eps$-skyline and that the estimates in $M$ are accurate.

\begin{proof}
From \Cref{clm:correctness}, and the fact that $E$ occurs, we know that $S$ is an $\eps$-skyline of $A$.  

First, we shall see that the changes \Cref{algo:truncate} makes to $S$ do not invalidate either \Cref{lem:nofalseneg} or \Cref{lem:nofalsepos}, which means \Cref{clm:correctness} still holds.   Observe that when line 5 of \Cref{algo:truncate} is true, every removed index $s'$ is ensured that the $s\in S$ to its left indexes an $\eps$-better arm. This follows from identical algebra that led to \eqref{eq:stepdown}: substitute $prev$ with $s$ and $k_{\ell,m}$ with $s'$.  Because every removed index $s'\notin S$ is ensured that the $s\in S$ to its left indexes an $\eps$-better arm, \Cref{lem:nofalseneg} still holds.  \Cref{lem:nofalsepos} is unaffected as \Cref{algo:truncate} does not add any arms.

Finally, we will see that only $O(1/\eps)$ arms remain in $S$ after running \Cref{algo:truncate}.  We know $\hm[S[i]] + 3\eps/4 \leq \hm[S[i+1]]$ holds for all $i\in [1,|S|]$. From identical algebra that led to \eqref{eq:stepup}, every $S[i+1]$ has mean $\mu[S[i+1]]$ at least $\frac{7\eps}{12}$ more than $\mu[S[i]]$. Because the mean of each arm is within $[0,1]$, there can be at most $\frac{12}{7\eps} = O(\frac{1}{\eps})$ indices in $S$.
\end{proof}

\section{A Sample Complexity Lower Bound for $\eps$-Skyline Identification}

\newcommand{\alg}{\mathcal{A}}
\newcommand{\gameone}{\mathsf{id1}}
\newcommand{\gamemany}{\mathsf{idT}}
\newcommand{\bernoulli}{\mathit{Ber}}
\newcommand{\View}{V}
\newcommand{\view}{v}

Notice that the bottleneck in the sample complexity of our algorithm is that, in the first stage, we have to solve $\approx \frac{1}{\eps}$ independent $\eps$-best-arm-identification problems.  In order to afford the union bound, we needed to solve each of these problems successfully with probability at least $1-\frac{\eps}{\delta}$, which (by a slight strengthening of~\cite{MannorTsi}) requires $O(\frac{n}{\eps^2}\log \frac{1}{\eps \delta})$ samples.  We then show that, when $\eps > \frac{1}{n}$, any algorithm for solving the $\eps$-Pareto-optimal-arm-identification problem must solve $\Omega(\frac{1}{\eps})$ independent $\eps$-best-arm identification problems, thereby obtaining our lower bound.

In order to obtain our lower bound, we will need to use some of the specific properties of the hard instances for best-arm identification that were used in the lower bound of~\cite{MannorTsi}.  Specifically, implicit in their proof is the following theorem:
\begin{thm}[Implicit in~\cite{MannorTsi}] \label{thm:MTLB}
Fix any $p \in [\frac14,\frac34]$.  Let $\alg$ be an algorithm, and define the following game $\gameone_{n,p,\eps}(\alg)$:
\begin{enumerate}

\vspace{-1mm}
\item Choose a random $c \in [n]$

\vspace{-1mm}
\item Define $A = \{A[1],\dots,A[n]\}$ so that $A[c] \sim \bernoulli(p+2\eps)$ and for $i \neq c$, $A[i] \sim \bernoulli(p)$.

\vspace{-1mm}
\item $\alg$ draws samples from the arms $A$

\vspace{-1mm}
\item $\alg$ outputs a guess $c^*$

\end{enumerate}
Then there exists a function $S(n,\eps,\delta) = \Omega(\frac{n}{\eps^2} \log \frac{1}{\delta})$ such that if $\alg$ draws fewer than $S(n,\eps,\delta)$ samples in expectation,
$$
\pr{\gameone_{n,p,\eps}(\alg)}{c^* = c} < 1-\delta.
$$
\end{thm}
Note that the unique $\eps$-best-arm is $c$, therefore any algorithm that has at least a $1-\delta$ probability of identifying an $\eps$-best-arm on this sort of instance must draw at least $S(n,\eps,\delta)$ samples in expectation.

The next lemma is a simple ``direct-product lemma'' for this game, which roughly asserts that any algorithm $\alg$ that is asked to solve $T$ independent copies of this game, and succeeds in \emph{all} of these copies with probability at least $1-\delta$, must draw nearly $T \cdot S(n,\eps,\delta/T)$ samples, which means that that the strategy of solving each copy independently with probability $1-\delta/T$ and suffering a union bound over the $T$ copies is essentially optimal.
\begin{lem} \label{lem:lbmany}
Fix any $T \in \mathbb{N}$ and any $\vec{p} = (p_1,\dots,p_T) \in [\frac14,\frac34]^T$.  Let $\alg$ be an algorithm, and define the following game $\gamemany_{n,\vec{p},\eps,T}(\alg)$:
\begin{enumerate}

\vspace{-1mm}
\item Choose independent random $\vec{c} = (c_1,\dots,c_T) \in [n]^T$

\vspace{-1mm}
\item For every $t \in [T]$, define a set of arms $A_t = \{A_t[1],\dots,A_t[n]\}$ so that $A[c_t] \sim \bernoulli(p_t+2\eps)$ and for $i \neq c_t$, $A[i] \sim \bernoulli(p_t)$.  Let $A = A_1 \cup \dots \cup A_T$.

\vspace{-1mm}
\item $\alg$ draws samples from the arms $A$

\vspace{-1mm}
\item $\alg$ outputs a guess $\vec{c}^*$

\end{enumerate}
Then if $\alg$ draws fewer than $\frac{T}{2} \cdot S(n,\eps, \sqrt[3]{256\delta/T}) = \Omega(\frac{Tn}{\eps^2} \log \frac{T}{\delta})$ samples in expectation,
$$
\pr{\gamemany_{n,\vec{p},\eps,T}(\alg)}{\vec{c}^* = \vec{c}} < 1-\delta.
$$
\end{lem}
As above, note that for every set of arms $A_t$, the unique best arm is $c_t$, and therefore the lower bound on the number of samples applies to any algorithm $\alg$ that solves $\eps$-best-arm identification simultaneously on all sets $A_1,\dots,A_T$.  We now give a high level sketch of why this direct product result holds.
\begin{proof}[Proof]
Note that the lemma would be immediate if we could argue that the probability that $\alg$ succeeds in identifying each value $c_t$ is independent across the $T$ copies, but this is unfortunately not true.  However, suppose we let the random variable $\View$ denote the entire sequence of samples and values obtained by $\alg$ throughout the course of the algorithm.  Then for every realization $\view$, the random variables $(c_t \mid \View = \view)$ are mutually independent.  This claim can be proven by induction on the number of samples made by $\alg$---initially all values $c_t$ are chosen independently, and in each step the response received by $\alg$ depends on only a single value $c_t$.

To simplify arithmetic, let $p := \sqrt[3]{256 \delta / T}$ denote the parameter from the theorem statement.  Now, suppose we have an algorithm $\alg$ that draws fewer than $\frac{T}{2}\cdot S(n,\eps,p)$ samples in expectation.  By linearity of expectation and Markov's inequality, there is a set $U \subseteq [T]$ of size at least $T/2$ such that for every $t \in U$, the expected number of samples $\alg$ draws from the set of arms $A_t$ is at most $S(n,\eps,p)$.  By Theorem~\ref{thm:MTLB}, we can conclude
$$
\forall t \in U~~~\pr{\gamemany_{n,\vec{p},\eps,T}(\alg)}{c_t^* \neq c_t} > p.
$$
From this, and two applications of Markov's inequality we obtain the following: there exists a subset of views $V'$ such that $\pr{}{V'} \geq \frac{p}{4}$ and a collection of subsets $\{U'_v \subseteq U\} \subset U$, each of size at least $|U| \cdot \frac{p}{8} = \frac{pT}{16}$ such that
\begin{align*}
\forall v \in V', \forall t \in U'_v~~~\pr{\gamemany_{n,\vec{p},\eps,T}(\alg)}{c^*_t \neq c_t \mid \view} > p/4
\end{align*}
Using the the mutual independence of the variables $(c_t \mid \view)$, we obtain
$$
\forall v \in V'~~~\pr{\gamemany_{n,\vec{p},\eps,T}(\alg)}{\exists t \in U'_v~~c^*_t \neq c_t \mid \view} > 1 - (1-p/4)^{|U'|} \geq 1 - e^{-p^2T/32}
$$
Since $\exists t \in U'_v~~c^*_t \neq c_t$ implies that $\vec{c}^* \neq \vec{c}$, we have
$$
\forall v \in V'~~~\pr{\gamemany_{n,\vec{p},\eps,T}(\alg)}{\vec{c}^* \neq \vec{c} \mid \view} > 1 - e^{-p^2T/32}
$$
Using the fact that $\pr{}{V'} \geq \frac{p}{4}$, we have
$$
\pr{\gamemany_{n,\vec{p},\eps,T}(\alg)}{\vec{c}^* \neq \vec{c}} > \frac{p}{4}\left(1 - e^{-p^2T/32}\right) \geq \frac{p^3 T}{256} = \delta
$$
where the final inequality uses the fact that $p^3 T$ is smaller than some absolute constant so that we can write $1-e^{-p^2T/32} \geq p^2T/64$ and the final equality is by substituting our choice of $p$.  This completes our proof sketch.
\end{proof}

Finally, we can show that, for some $T = \Omega(1/\eps)$, and some choice of $\vec{p} = (p_1,\dots,p_T)$, and $\eps$-sykline identification algorithm can be used to ``win'' the game $\gamemany_{n,\vec{p},\eps,T}$, from which we immediately obtain a lower bound on the sample complexity of $\eps$-skyline identification.

\begin{thm}
Fix any $n \in \mathbb{N}$ and fix $\eps, \delta > 0$ smaller than some absolute constants.  Suppose there is an algorithm $\alg$ that, with probability at least $1-\delta$, identifies an $\eps$-skyline on $n$ arms.  Then in expectation $\alg$ draws at least $\frac{1}{4\eps} \cdot S(2 \eps n, \eps, \sqrt[3]{128 \delta \eps}) = \Omega( \frac{n}{\eps^2} \log \frac{1}{\delta \eps})$ samples.
\end{thm}
\begin{proof}
Define $T = \frac{1}{2\eps}$ and $m = 2\eps n$ so that $n = Tm$.  For notational simplicity, we assume $T$ and $m$ are integers.  Define
$$
\vec{p} := \left(\frac14,\frac14 + 2\eps ,\frac14 + 4\eps,\dots,\frac34\right)
$$
By construction, $\vec{p} \in [\frac14, \frac34]^{T}$.  By Lemma~\ref{lem:lbmany}, if $\alg$ is an algorithm such that
\begin{equation} \label{eq:winmany}
\pr{\gamemany_{m, \vec{p}, \eps, T}}{\vec{c} = \vec{c}^*} \geq 1-\delta,
\end{equation}
then $\alg$ draws $\frac{T}{2} \cdot S(m, \eps, \sqrt[3]{256 \delta /T}) = \Omega(\frac{n}{\eps^2} \log \frac{1}{\eps \delta})$ samples in expectation.

Thus, all that remains is to show that, given any algorithm $\alg'$ that identifies an $\eps$-skyline on $n$ arms with probability at least $1-\delta$, we can construct an algorithm $\alg$ satisfying~\eqref{eq:winmany}.  We can construct such an algorithm as follows:
\begin{enumerate}
\vspace{-1mm}
\item Generate a vector $\vec{c}$ and generate the sets $A_1,\dots,A_T$ as in $\gamemany_{m, \vec{p}, \eps, T}$.  Then generate a set of $n = Tm$ arms $A'$ by concatenating the sets $A_1,\dots,A_T$ together in order.  Specifically, for $t \in [T], i \in [m]$, let $A'[(t-1)T + i] = A_{t}[i]$.

\vspace{-1mm}
\item Obtain $S\subseteq [n]$ by running $\alg'$ on the set of arms $A'$

\vspace{-1mm}
\item For every $t\in[T]$, let $A'_t = \{(t-1)m + 1, \dots, tm\}$ and let $c^*_t = \max S \cap A'_t$ be the largest index of any arm in the skyline that lies in $A'_t$.  Note that $A[c^*_t]$ corresponds to the arm $A_{t}[j~\mathrm{mod}~m]$.

\vspace{-1mm}
\item Output $\vec{c}^* = (c^*_1,\dots,c^*_T)$
\end{enumerate}

To complete the proof, we will show that if $S$ is an $\eps$-skyline for $A'$ then $\vec{c}^* = \vec{c}$.  First, by construction, $c_t$ denotes the unique $\eps$-best arm in $A_{t}$, and by construction of $\vec{p}$, it is also better than all arms in $A_{1},\dots,A_{t-1}$.  When we construct the arms $A'$, $A_t$ becomes the set of arms $A'_t$ and similarly $c_t$ becomes the arm $(t-1)m + c_t$.  By construction of $\vec{p}$, arm $(t-1)m + c_t$ in $A'$ is the best arm among the set $\{1,\dots,tm\}$.  By the first condition of Definition \ref{def:runningmax}, any valid $\eps$-skyline $S$ must contain each arm $(t-1)n + c_t$.  Similarly, by the second condition of Definition~\ref{def:runningmax}, arms in $A'$ that are in the set $\{(t-1)m+c_t + 1,tm]$ cannot appear in $S$ because they are dominated by arm $(t-1)m+c_t$.  Finally, because $(t-1)m + c_t \in S \cap A'_{t}$, and no larger element of $A' \cap A'_{t}$ is in $S$, we have $c^*_{t} = c_{t}$ for every $t$.  This completes the proof.
\end{proof}

\section*{Acknowledgements} We are grateful to Jay Aslam, Maryam Aziz, Bobby Kleinberg, and Alex Slivkins for many helpful discussions related to this work.

\bibliographystyle{alpha}
\bibliography{sample}

\newcommand{\etalchar}[1]{$^{#1}$}
\begin{thebibliography}{EDMM06}

\bibitem[ACOD16]{Auer}
Peter Auer, Chao-Kai Chiang, Ronald Ortner, and Madalina Drugan.
\newblock Pareto front identification from stochastic bandit feedback.
\newblock In {\em Artificial Intelligence and Statistics}, 2016.

\bibitem[BKS01]{BKS01}
Stephan Borzsony, Donald Kossmann, and Konrad Stocker.
\newblock The skyline operator.
\newblock In {\em International Conference on Data Engineering {(ICDE)}}. IEEE,
  2001.

\bibitem[CGL16]{CGL16}
Lijie Chen, Anupam Gupta, and Jian Li.
\newblock Pure exploration of multi-armed bandit under matroid constraints.
\newblock In {\em Conference on Learning Theory}, 2016.

\bibitem[CGL{\etalchar{+}}17]{CGLQW17}
Lijie Chen, Anupam Gupta, Jian Li, Mingda Qiao, and Ruosong Wang.
\newblock Nearly optimal sampling algorithms for combinatorial pure
  exploration.
\newblock {\em arXiv preprint arXiv:1706.01081}, 2017.

\bibitem[EDMM06]{Even-Dar}
Eyal Even-Dar, Shie Mannor, and Yishay Mansour.
\newblock Action elimination and stopping conditions for the multi-armed bandit
  and reinforcement learning problems.
\newblock {\em Journal of Machine Learning Research}, 2006.

\bibitem[GGLB11]{GGLB11}
Victor Gabillon, Mohammad Ghavamzadeh, Alessandro Lazaric, and S{\'e}bastien
  Bubeck.
\newblock Multi-bandit best arm identification.
\newblock In {\em Advances in Neural Information Processing Systems}, 2011.

\bibitem[Kle06]{Kleinberg06}
Robert~D. Kleinberg.
\newblock Anytime algorithms for multi-armed bandit problems.
\newblock In {\em Proceedings of the 17th Annual {ACM-SIAM} Symposium on
  Discrete Algorithms {(SODA)}}, 2006.

\bibitem[KS10]{KS10}
Shivaram Kalyanakrishnan and Peter Stone.
\newblock Efficient selection of multiple bandit arms: Theory and practice.
\newblock In {\em Proceedings of the 27th International Conference on Machine
  Learning {(ICML)}}, 2010.

\bibitem[KTAS12]{KTAS12}
Shivaram Kalyanakrishnan, Ambuj Tewari, Peter Auer, and Peter Stone.
\newblock Pac subset selection in stochastic multi-armed bandits.
\newblock In {\em Proceedings of the 29th International Conference on Machine
  Learning {(ICML)}}, 2012.

\bibitem[MT04]{MannorTsi}
Shie Mannor and John~N. Tsitsiklis.
\newblock The sample complexity of exploration in the multi-armed bandit
  problem.
\newblock {\em Journal of Machine Learning Research}, 2004.

\end{thebibliography}

\end{document}